\pdfoutput=1

\documentclass[11pt]{article}

\usepackage{emnlp2021}

\usepackage{times}
\usepackage{latexsym}

\usepackage[T1]{fontenc}

\usepackage[utf8]{inputenc}

\usepackage{microtype}

\usepackage{amsmath}
\usepackage{amssymb}

\usepackage{physics}

\usepackage{amsthm}
\newtheorem{theorem}{Theorem}[section]
\newtheorem{definition}{Definition}[section]
\newtheorem{corollary}{Corollary}[section]

\usepackage[short]{optidef}

\newif\ifarxiv
\arxivtrue

\title{When differential privacy meets NLP: The devil is in the detail}

\author{Ivan Habernal \\
	Trustworthy Human Language Technologies  \\
	Department of Computer Science \\
	Technical University of Darmstadt \\
	\texttt{ivan.habernal@tu-darmstadt.de} \\
	\texttt{www.trusthlt.org}  \\}

\date{}

\begin{document}
	
	\ifarxiv
	\onecolumn
	\noindent \textbf{When differential privacy meets NLP: The devil is in the detail}
	
	\medskip
	\noindent Ivan Habernal
	
	\bigskip
	This is a \textbf{camera-ready version} of the article accepted for publication at the \emph{2021 Conference on Empirical Methods in Natural Language Processing (EMNLP 2021)}. The final official version will be published on the ACL Anthology website later in 2021: \url{https://aclanthology.org/}
	
	\medskip
	Please cite this pre-print version as follows.
	\medskip
	
\begin{verbatim}
@InProceedings{Habernal.2021.EMNLP,
    title = {{When differential privacy meets NLP:
              The devil is in the detail}},
    author = {Habernal, Ivan},
    publisher = {Association for Computational Linguistics},
    booktitle = {Proceedings of the 2021 Conference on Empirical
                 Methods in Natural Language Processing},
    pages = {(to appear)},
    year = {2021},
    address = {Punta Cana, Dominican Republic}
}
\end{verbatim}
	\twocolumn
	\fi

\maketitle
\begin{abstract}
Differential privacy provides a formal approach to privacy of individuals.
Applications of differential privacy in various scenarios, such as protecting users' original utterances, must satisfy certain mathematical properties.
Our contribution is a formal analysis of ADePT, a differentially private auto-encoder for text rewriting \citep{Krishna.et.al.2021.EACL}. ADePT achieves promising results on downstream tasks while providing tight privacy guarantees.
Our proof reveals that ADePT is not differentially private, thus rendering the experimental results unsubstantiated. We also quantify the impact of the error in its private mechanism, showing that the true sensitivity is higher by at least factor 6 in an optimistic case of a very small encoder's dimension and that the amount of utterances that are not privatized could easily reach 100\% of the entire dataset. Our intention is neither to criticize the authors, nor the peer-reviewing process, but rather point out that if differential privacy applications in NLP rely on formal guarantees, these should be outlined in full and put under detailed scrutiny.

\end{abstract}

\section{Introduction}

The need for NLP systems to protect individuals' privacy has led to the adoption of differential privacy (DP). DP methods formally guarantee that the output of the algorithm will be `roughly the same' regardless of whether or not any single individual is present in the central dataset; this is achieved by employing randomized algorithms \citep{Dwork.Roth.2013}. Local DP, a variant of DP, mitigates the need for a central dataset and applies randomization on each individual's datapoint. Local DP thus guarantees that its output for an individual A will be `almost indistinguishable' from the output of any other individuals B or C.\footnote{See the \emph{randomized response} for an easy explanation of local DP for a single bit \citep{Warner.1965}.}

This level of privacy protection makes local DP an ideal framework for NLP applications that operate on sensitive user input which should not be collected and processed globally by an untrusted party, e.g., users' verbatim utterances. When the utterances are `privatized' by local DP, any future post-processing or adversarial attack cannot reveal more than allowed by the particular local DP algorithm's properties (namely the $\varepsilon$ parameter; see later Sec.~\ref{sec:theoretical.background}).

ADePT, a local DP algorithm recently published at EACL by \citet{Krishna.et.al.2021.EACL} from Amazon Alexa, proposed a differentially private auto-encoder for text rewriting.
In summary, ADePT takes an input textual utterance and re-writes it in a way such that the output satisfies local DP guarantees. Unfortunately, a thorough formal ana\-ly\-sis reveals that ADePT is in fact not differentially private and the privatized data do not protect privacy of individuals as formally promised.

In this short paper, we shed light on ADePT's main argument, the privacy mechanism.
We briefly introduce key concepts from differential privacy (DP) and present a detailed proof of the Laplace mechanism (Sec.~\ref{sec:theoretical.background}).
Section~\ref{sec:adept} introduces ADePT's \citep{Krishna.et.al.2021.EACL} architecture and its main privacy argument. We formally prove that the proposed ADePT's mechanism is in fact not differentially private (Sec.~\ref{sec:adept.is.not.dp}) and determine the actual sensitivity of its private mechanism (Sec.~\ref{sec:sensitivity.of.adept}). We sketch to which extent ADePT breaches privacy as opposed to the formal DP guarantees (Sec.~\ref{sec:non-protected}) and discuss a potential adversary attack (Appendix~\ref{sec:potential.attacks}).

\section{Theoretical background}
\label{sec:theoretical.background}

From a high-level perspective, DP works with the notion of \emph{individuals} whose information is contained in a \emph{database} (\emph{dataset}). Each individual's \emph{datapoint} (or \emph{record}), which could be a single bit, a number, a vector, a structured record, a text document, or any arbitrary object, is considered private and cannot be revealed. Moreover, even whether or not any particular individual A is in the database is considered private.

\begin{definition}
	\label{def:neighboring.datasets}
	Let $\mathcal{X}$ be a `universe' of all records and $x,y \in \mathcal{X}$ be two datasets from this universe. We say that $x$ and $y$ are neighboring datasets if they differ in one record.
\end{definition}

For example, let dataset $x$ consist of $|x|$ documents where each document is associated with an individual whose privacy we want to preserve. Let $y$ differ from $x$ by one document, so either $|y| = |x| \pm 1$, or $|y| = |x|$ with $i$-th document replaced. Then by definition \ref{def:neighboring.datasets}, $x$ and $y$ are neighboring datasets.

\paragraph{Global DP and queries}

In a typical setup, the database is not public but held by a \emph{trusted curator}. Only the curator can fully access all datapoints and answer any \emph{query} we might have, for example how many individuals are in the database, whether or not B is in there, what is the most common disease (if the database is medical), what is the average length of the documents (if the database contains texts), and so on. The types of queries are task-specific, and we can see them simply as functions with arbitrary domain $\mathcal{X}$ and co-domain $\mathcal{Z}$. In this paper, we focus on a simple query type, the \emph{numerical query}, that is a function with co-domain in $\mathbb{R}^n$.

For example, consider a dataset $x \in \mathcal{X}$ containing textual documents and a numerical query $f: \mathcal{X} \to \mathbb{R}$ that returns an average document length. Let's assume that the length of each document is private, sensitive information. Let the dataset $x$ contain a particular individual A whose privacy we want to breach. Say we also have some leaked background information, in particular a neighboring dataset $y \in \mathcal{X}$ that contains all datapoints from $x$ except for A. Now, if the trusted curator returned the true value of $f(x)$, we could easily compute A's document length, as we know $f(y)$, and thus we could breach A's privacy. To protect A's privacy, we will employ randomization.

\begin{definition}
	Randomized algorithm $\mathcal{M} : \mathcal{X} \to \mathcal{Z}$ takes an input value $x \in \mathcal{X}$ and outputs a value $z \in \mathcal{Z}$ nondeterministically, e.g., by drawing from a certain probability distribution.
\end{definition}

Typically, randomized algorithms are parameterized by a density (for $z \in \mathbb{R}^n$) or a discrete distribution (for categorical or binary $z$). The randomized algorithm `perturbs' the input by drawing from that distribution. We suggest to consult \citep{Igamberdiev.Habernal.2021} for yet another NLP introduction to differential privacy.

\begin{definition}
	\label{def:dp}
	Randomized algorithm $\mathcal{M}$ satisfies ($\varepsilon$,0)-differential privacy if and only if for any neighboring datasets $x, y \in \mathcal{X}$ from the domain of $\mathcal{M}$, and for any possible output $z \in \mathcal{Z}$ from the range of $\mathcal{M}$, it holds
	\begin{equation}
	\label{eq:dp-def}
	\mathrm{Pr}[\mathcal{M}(x) = z] \leq \exp(\varepsilon) \cdot \mathrm{Pr}[\mathcal{M}(y) = z]
	\end{equation}
	where $\mathrm{Pr}[.] $ denotes probability\footnote{The definition holds both for densities $p$ and probability mass functions $P$ as $\mathrm{Pr}$.} and $\varepsilon \in \mathbb{R}^+$ is the privacy budget. A smaller $\varepsilon$ means stronger privacy protection, and vice versa \citep{Wang.et.al.2020.Sensors,Dwork.Roth.2013}.
\end{definition}

In words, to protect each individual's privacy, DP adds randomness when answering queries such that the query results are `similar' for any pair of neighboring datasets. For our example of the average document length, the true average length would be randomly `noisified'.

Another view on $(\varepsilon,0)$-DP is when we treat $\mathcal{M}(x)$ and $\mathcal{M}(y)$ as two probability distributions. Then $(\varepsilon,0)$-DP puts upper bound $\varepsilon$ on Max Divergence $\mathbb{D}_{\infty}(\mathcal{M}(x) || \mathcal{M}(y))$, that is the maximum `difference' of any output of two random variables.\footnote{
	$\mathbb{D}_{\infty}(\mathcal{M}(x) || \mathcal{M}(y)) = \max_{z \in \mathcal{Z}} \left\{ \ln \frac{ \mathrm{Pr}[\mathcal{M}(x) = z]  }{\mathrm{Pr}[\mathcal{M}(y) = z]} \right\}$}

Differential privacy has also a Bayesian interpretation, which compares the adversary's prior with the posterior after observing the values. The odds ratio is bounded by $\exp(\varepsilon)$, see \citep[p.~266]{Mironov.2017.CSF}.

\paragraph{Neighboring datasets and local DP}

The original definition of neigboring datasets (Def.~\ref{def:neighboring.datasets}) is usually adapted to a particular scenario; see \citep{Desfontaines.Pejo.2020} for a thorough overview. So far, we have shown the global DP scenario with a trusted curator holding a database of $|x|$ individuals. The size of the database can be arbitrary, even containing a single individual, that is $|x| = 1$. In this case, we say a dataset $y \in \mathcal{X}$ is neighboring if it contains another single individual ($y \in \mathcal{X}$, $|y| = 1$). This setup allows us to proceed without the trusted curator, as each individual queries its single record and returns differentially private output; this scenario is known as \emph{local} DP.

In local differential privacy, where there is no central database of records, \emph{any pair of data points (examples, input values, etc.) is considered neighboring} \citep{Wang.et.al.2020.Sensors}. This also holds for ADePT: using the DP terminology, any two utterances $x$, $y$ are neighboring datasets \citep{Krishna.et.al.2021.EACL}.

\begin{definition}
\label{def:sensitivity}
Let $x,y \in \mathcal{X}$ be neighboring datasets. The $\ell_1$-sensitivity of a function $f : \mathcal{X} \rightarrow \mathbb{R}^n$ is defined as
\begin{equation}
\label{eq:sensitivity.def}
\Delta f = \max_{x, y} \norm{f(x) - f(y)}_1
\end{equation}
where $\norm{.}_1$ is a $\ell_1$-norm defined as $\norm{\mathbf{x}}_1 = \sum_{i = 1}^{n} \lvert x_i \rvert$ \citep[p.~31]{Dwork.Roth.2013}.
\end{definition}

\begin{definition} Laplace density with scale $b$ centered at $\mu$ is defined as
\begin{equation}
\label{eq:laplace}
\mathrm{Lap}(t; \mu, b) = \frac{1}{2b} \exp(- \frac{\lvert \mu - t\rvert}{b})
\end{equation}
\end{definition}

\begin{definition}
\label{def:laplace.mechanism}
Laplace randomized algorithm \citep[p.~32]{Dwork.Roth.2013}. Given any function $f : \mathcal{X} \rightarrow \mathbb{R}^n$, the Laplace mechanism is defined as
\begin{equation}
\label{eq:laplace.mechanism}
\mathcal{M}_{L}(x, f, \varepsilon) = f(x) + (Y_1, \dots, Y_n)
\end{equation}
where $Y_i$ are i.i.d. random variables drawn from a Laplace distribution
\begin{equation}
\label{eq:laplace.mechnism.draw}
Y_i \sim \mathrm{Lap}\left(\mu = 0; b = \Delta f/\varepsilon \right)
\end{equation}
\end{definition}

An analogous definition centers the Laplace noise directly at the function's output, that is

\begin{equation}
\label{eq:laplace.draw2}
\begin{aligned}
\mathcal{M}_{L} = ( &Y_i \sim \mathrm{Lap}(\mu = f(x)_1; b = \Delta f/\varepsilon), \\
&\dots, \\
&Y_n \sim \mathrm{Lap}(\mu = f(x)_n; b = \Delta f/\varepsilon) )
\end{aligned}
\end{equation}

From Definition \ref{def:laplace.mechanism} also immediately follows that at point $z \in \mathbb{R}^n$, the density value of the Laplace mechanism $p(M_L(x, f, \varepsilon) = z)$ is

\begin{equation}
\label{eq:laplace.mechanism.pdf.full}
\prod_{i = 1}^{n} \frac{\varepsilon}{2 \Delta f } \exp \left( - \frac{ \varepsilon \abs{ f(x)_i - z_i}}{\Delta f } \right)
\end{equation}

\begin{theorem}
\label{theorem.laplace.dp}
The Laplace randomized algorithm preserves $(\varepsilon, 0)$-DP \citep{Dwork.Roth.2013}.
\end{theorem}

As ADePT relies on the proof of the Laplace mechanism, we show the full proof in Appendix~\ref{app:laplace}.

\section{ADePT by \citet{Krishna.et.al.2021.EACL}}
\label{sec:adept}

Let $\mathbf{u}$ be an input text (a sequence of words or a vector, for example; this is not key to the main argument). $\mathsf{Enc}$ is an encoder function from input $\mathbf{u}$ to a latent representation vector $\mathbf{r} \in \mathbb{R}^n$ where $n$ is the number of dimensions of that latent space. $\mathsf{Dec}$ is a decoder from the latent representation back to the original input space (again, a sequence of words or a vector). What we have so far is a standard auto-encoder, such that

\begin{equation}
\label{eq:1}
\mathbf{r} = \mathsf{Enc}(\mathbf{u}) \quad \text{and} \quad \mathbf{v} = \mathsf{Dec}(\mathbf{r}).
\end{equation}

\citet{Krishna.et.al.2021.EACL} define ADePT as a randomized algorithm that, given an input $\mathbf{u}$, generates $\mathbf{v}$ as $\mathbf{v} = \mathsf{Dec}(\mathbf{r'})$, where $\mathbf{r'} \in \mathbb{R}^n$ is a clipped latent representation vector with added noise

\begin{equation}
\label{eq:krishna.2}
\mathbf{r'} = \mathbf{r} \cdot \min \left( 1, \frac{C}{\norm{\mathbf{r}}_2} \right) + \eta
\end{equation}

where $\eta \in \mathbb{R}^n$, $C \in \mathbb{R}$ is an arbitrary clipping constant, and $\norm{\mathbf{.}}_2$ is an $\ell_2$ (Euclidean) norm defined as $\norm{\mathbf{x}}_2 = \sqrt{\sum_{i = 1}^{n} x_i^2}$.

\begin{theorem}[which is false]
\label{theorem:krishna.dp}
\citep{Krishna.et.al.2021.EACL}	If $\eta$ is a multidimensional noise, such that each element $\eta_i$ is independently drawn from a distribution shown in equation \ref{eq:krishna.3}, then the transformation from $\mathbf{u} \to \mathbf{v}'$ is $(\varepsilon, 0)$-DP.
\end{theorem}

\begin{equation}
\label{eq:krishna.3}
\mathrm{Lap}(\eta_i) \sim \frac{\varepsilon}{4C} \exp \left( - \frac{\varepsilon \lvert v_i \rvert }{2 C} \right)
\end{equation}

\begin{proof}
\citet{Krishna.et.al.2021.EACL} refers to the proof of Theorem 3.6 by \citet[p.~32]{Dwork.Roth.2013}, which is the proof of the Laplace mechanism.
\end{proof}

First, $v_i$ in Eq.~\ref{eq:krishna.3} is ambiguous as it `semantically' relates to $\mathbf{v}$ which is the decoded vector that comes first \emph{after} drawing a random value; moreover $\eta$ and $\mathbf{v}$ have different dimensions. Given that the authors employ Laplacian noise and base their proofs on Theorem~3.6 from \citet[p.~32]{Dwork.Roth.2013}, we believe that Eq.~\ref{eq:krishna.3} is the standard Laplace mechanism

\begin{equation}
\label{eq:eq-eta-lap}
\eta_i \sim \mathrm{Lap}\left(\mu = 0; b = \Delta f/\varepsilon \right),
\end{equation}

such that each value $\eta_i$ is drawn independently from a zero-centered Laplacian noise parametrized by scale $b$ (Definition \ref{def:laplace.mechanism}). Given the density from Eq.~\ref{eq:laplace}, we rewrite Eq.~\ref{eq:eq-eta-lap} as

\begin{equation}
\label{eq:eta-laplace-2}
\eta_i \sim \frac{\varepsilon}{2 \Delta f} \exp(- \frac{\varepsilon \lvert t \rvert}{\Delta f}),
\end{equation}

\citet{Krishna.et.al.2021.EACL} set their clipped encoder output as the function $f$, that is\footnote{We contacted the authors several times to double check that this formula is correct without a potential typo but got no response. However other parts of the paper give evidence it is correct, e.g., the authors use an analogy to a hyper-sphere which is considered euclidean by default.}

\begin{equation}
\label{eq:f}
f = \mathbf{r} \cdot \min \left( 1, \frac{C}{\norm{\mathbf{r}}_2} \right).
\end{equation}

\begin{theorem}[which is false]
\label{theorem.krishna.sensitivity}
\citep{Krishna.et.al.2021.EACL} Let $f: \mathbb{R}^n \to \mathbb{R}^n$ be a function as defined in equation \ref{eq:f}. The $\ell_1$-sensitivity $\Delta f$ of this function is $2C$.
\end{theorem}

\begin{proof}
\citep{Krishna.et.al.2021.EACL} Maximum $\ell_1$ norm difference between two points in a hyper-sphere of radius $C$ is $2C$.
\end{proof}

Thus by plugging the sensitivity $\Delta f$ from Theorem \ref{theorem.krishna.sensitivity} into Eq.~\ref{eq:eta-laplace-2}, we obtain
\begin{equation}
\label{eq:krishna-lap-final}
\eta_i \sim \frac{\varepsilon}{4C} \exp(- \frac{\varepsilon \lvert t \rvert}{2C}),
\end{equation}
which is what \citet{Krishna.et.al.2021.EACL} express in Eq.~\ref{eq:krishna.3}.
To sum up, the essential claim of \citet{Krishna.et.al.2021.EACL} is that if each $\eta_i$ is drawn from Laplacian distribution with scale $\frac{2C}{\varepsilon}$, their mechanism is $(\varepsilon, 0)$ differentially private.

\section{ADePT with Laplace mechanism is not differentially private}
\label{sec:adept.is.not.dp}

\begin{proof} Following the proof of Theorem~\ref{theorem.laplace.dp}, the following bound (Eq.~\ref{eq:laplace.proof.key-part}) must hold for any $x, y$ 

\begin{small}
\begin{equation*}
\frac{p(M_L(x, f, \varepsilon) = z)}{p(M_L(y, f, \varepsilon) = z)} \leq \exp \left( \frac{ \varepsilon}{\Delta f} \cdot \norm{f(y) - f(x)}_1 \right)
\end{equation*}
\end{small}
and thus this inequality must hold too
\begin{equation}
\label{eq:key.inequality.disproof}
\exp \left( \frac{ \varepsilon}{\Delta f} \cdot \norm{f(y) - f(x)}_1 \right) \leq \exp(\varepsilon)
\end{equation}

Fix the clipping constant $C > 0$ arbitrarily ($C \in \mathbb{R}$), set dimensions to $n = 2$. Let $\mathbf{r}_y = (\frac{2}{3}C, \frac{2}{3}C)$ be the input $y$ of the clipping function $f$ from Eq.~\ref{eq:f}.

\begin{align}
f(y) &= \mathbf{r}_y \cdot \min \left( 1, \frac{C}{\norm{\mathbf{r}_y}_2} \right) \nonumber & \text{(from Eq.~\ref{eq:f})} \\
&= \mathbf{r}_y \cdot \min \left( 1, \frac{C}{\frac{2\sqrt{2}}{3} C} \right) &\\
&= \mathbf{r}_y \cdot \min \left( 1, 1.06066... \right) &\\
&= \mathbf{r}_y \cdot 1 = \left( \frac{2C}{3}, \frac{2C}{3} \right) &
\end{align}

Similarly, let  $\mathbf{r}_x = (-\frac{2}{3}C, -\frac{2}{3}C)$ be input $x$, for which we get analogically $f(x) = (-\frac{2}{3}C, -\frac{2}{3}C)$. Then

\begin{align}
\norm{f(y) - f(x)}_1 = \\
\norm{\left(\frac{2C}{3}, \frac{2C}{3}\right) - \left(-\frac{2C}{3}, -\frac{2C}{3}\right)}_1  = \\
 =\frac{8C}{3} \label{eq:8c}
\end{align}

Plug Theorem~\ref{theorem.krishna.sensitivity} and Eq.~\ref{eq:8c} into Eq.~\ref{eq:key.inequality.disproof}

\begin{align}
\exp \left( \frac{ \varepsilon}{2C} \cdot \norm{f(y) - f(x)}_1 \right) &\leq \exp(\varepsilon) \\
\exp \left( \frac{ \varepsilon}{2C} \cdot \frac{8C}{3} \right) &\leq \exp(\varepsilon) \\
\exp \left( \frac{4}{3} \cdot \varepsilon \right) & \nleq \exp(\varepsilon)
\end{align}
therefore Theorem~\ref{theorem:krishna.dp} by \citet{Krishna.et.al.2021.EACL} must be false.
\end{proof}

In general, it is the inequality $\norm{\mathbf{x}}_1 \geq \norm{\mathbf{x}}_2$ that makes ADePT fail the DP proof.

\section{Actual sensitivity of ADePT}
\label{sec:sensitivity.of.adept}

\begin{theorem}
\label{theorem.krishna.sensitivity.disprove}
Let $f: \mathbb{R}^n \to \mathbb{R}^n$ be a function as defined in Eq.~\ref{eq:f}. The sensitivity $\Delta f$ of this function is $2C\sqrt{n}$.
\end{theorem}

\begin{proof}
See Appendix \ref{app:proof.of.sensitivity}.
\end{proof}

\begin{corollary}
Since $2C\sqrt{n} = 2C$ only for $n = 1$, ADePT could be differentially private only if the encoder's latent representation $\mathbf{r} = \mathsf{Enc}(\mathbf{u})$ were a single scalar.
\end{corollary}

Since \citet{Krishna.et.al.2021.EACL} do not specify the dimensionality of their encoder's output, we can only assume some typical values in a range from 32 to 1024, so that the true sensitivity of ADePT is $\approx 6$ to $32$ times higher than reported.

\section{Magnitude of non-protected data}
\label{sec:non-protected}

How many data points actually violate the privacy guarantees? Without having access to the trained model and its hyper-parameters ($C$, in particular), it is hard to reason about properties of the latent space, where privatization occurs. We thus simulated the encoder's `unclipped' vector outputs $\mathbf{r}$ by sampling 10k vectors from two distributions: 1) uniform within $(-C, +C)$ for each dimension, and 2) zero-centered normal with $\sigma^2 = 0.1 \cdot C$. Especially the latter one is rather optimistic as it samples most vectors close to zero. In reality these latent space vectors are unbounded.

Each pair of such vectors in the latent space after clipping but before applying DP (Eq.~\ref{eq:f}) is `neighboring datasets' so their $\ell_1$ distance must be bound by sensitivity ($2C$ as claimed in Theorem~\ref{theorem.krishna.sensitivity}) in order to satisfy DP with the Laplace mechanism.

We ran the simulation for an increasing dimensionality of the encoder's output and measured how many pairs violate the sensitivity bound.\footnote{Code available at \newline \url{https://github.com/habernal/emnlp2021-differential-privacy-nlp}} Fig.~\ref{fig:violation} shows the `curse of dimensionality' for norms. Even for a considerably small encoder's vector size of 32 and unbounded encoder's latent space, almost \textbf{none} of the data points would be protected by ADePT's Laplace mechanism.

\begin{figure}[h!]
\includegraphics[width=0.95\linewidth]{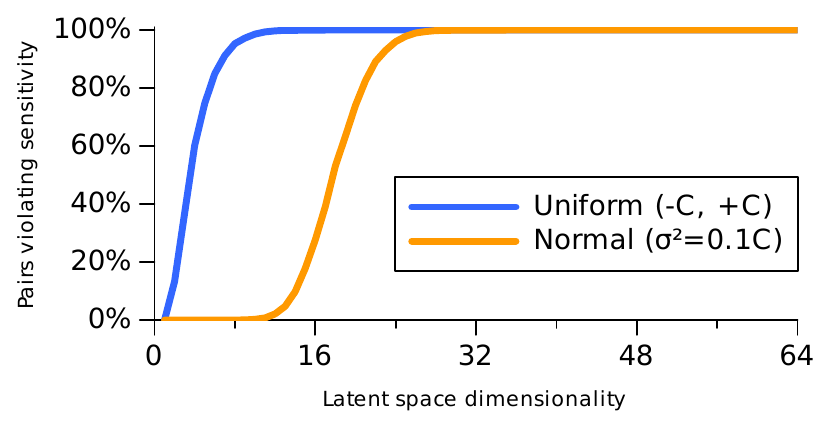}
\caption{\label{fig:violation} Simulation results. Percentage of `neighboring datasets' that violate the distance bounds required by the Laplace mechanism with sensitivity $2C$.}
\end{figure}

\section{Discussion}
\label{sec:discussion}

Local DP differs from centralized DP in such a way that there is no central database and once the privatized data item `leaves' an individual, it stays so forever. This makes typical membership inference attacks unsuitable, as no matter what happens to the rest of the world, the probability of inferring the individual's true value after observing their privatized data item is bounded by $\exp(\varepsilon)$.

For example, the ATIS dataset used in ADePT contains 5,473 utterances of lengths 1 to 46 tokens, with a quite limited vocabulary of 941 words. In theory, the search space of all possible utterances would be of size $941^{46} \approx 6\times 10^{136}$, and under $\varepsilon$-DP all of them are multiplicatively indistinguishable -- for example, after observing \emph{``on april first i need a ticket from tacoma to san jose departing before 7 am''} from ADePT's autoencoder privatized output, the true input might well have been \emph{``on april first i need a flight going from phoenix to san diego''} or \emph{``monday morning i would like to fly from columbus to indianapolis''} and our posterior certainty of any of those is limited by the privacy bound. However, since outputs of ADePT are leaking privacy, attacks are possible. We sketch a potential scenario in Appendix~\ref{sec:potential.attacks}.

There are two possible remedies for ADePT. Either the latent vector clipping in Eq.~\ref{eq:krishna.2} could use $\ell_1$-norm, or the Laplacian noise in Eq.~\ref{eq:krishna.3} could use the correct sensitivity as determined in Theorem~\ref{theorem.krishna.sensitivity.disprove}. In either case, the utility in the downstream tasks as presented by \citet{Krishna.et.al.2021.EACL} are expected to be worse due to a much larger amount of required noise.

\section{Conclusion}

This paper revealed a potential trap for NLP researchers when adopting a local DP approach.
We believe it contributes to a better understanding of the exact modeling choices involved in determining the sensitivity of local DP algorithms. We hope that DP will become a widely accessible and well-understood framework within the NLP community. 

\section*{Acknowledgements}

The independent research group TrustHLT is supported by the Hessian Ministry of Higher Education, Research, Science and the Arts.
Thanks to Max Glockner, Timour Igamberdiev, Jorge Cordona, Jan-Christoph Klie, and the anonymous reviewers for their helpful feedback.

\bibliography{references}
\bibliographystyle{acl_natbib}


\appendix

\section{Proof of Laplace mechanism}
\label{app:laplace}

\begin{theorem}
	\label{theorem:triangle.ineq}
	Negative triangle inequality for absolute values. For $a, x, y \in \mathbb{R}$,
	\begin{equation}
	\label{eq:negative.triangle.inequality}
	|x - a| - |y - a| \leq |x - y|.
	\end{equation}	
\end{theorem}

Proof is directly based on the triangle inequality.

\begin{corollary}
	\label{corollary:sensitivity}
	Definition~\ref{def:sensitivity} implies that $\Delta f$ is an upper bound value on the $\ell_1$ norm of the function output for any neighboring $x$ and $y$. In other words
	\begin{equation}
	\label{eq:sensitivity.upperbound}
	\norm{f(x) - f(y)}_1 \leq \Delta f
	\end{equation}
\end{corollary}

\begin{proof}[The actual proof \citep{Dwork.Roth.2013}] We will prove that for any $x, y$ the following ratio
	\begin{equation}
	\label{eq:proof.laplace.1}
	\frac{p(M_L(x, f, \varepsilon) = z)}{p(M_L(y, f, \varepsilon) = z)}
	\end{equation}
	is bounded by $\exp(\varepsilon)$ and thus satisfies Definition \ref{def:dp}. Fix $z \in \mathbb{R}^n$ arbitrarily. By plugging Eq.~\ref{eq:laplace.mechanism.pdf.full} into Eq.~\ref{eq:proof.laplace.1}, we get
	
	\begin{align}
	&= \prod_{i=1}^{n} \frac{ \frac{\varepsilon}{2\Delta f} \exp \left( - \frac{ \varepsilon \abs{f(x)_i - z_i} }{\Delta f } \right)}{ \frac{\varepsilon}{2\Delta f} \exp \left( - \frac{ \varepsilon \abs{f(y)_i - z_i} }{\Delta f } \right)}
	\\
	&= \prod_{i=1}^{n} \frac{  \exp \left(- \frac{\varepsilon}{\Delta f} \abs{ f(x)_i - z_i } \right)}{  \exp \left( - \frac{\varepsilon}{\Delta f} \abs{ f(y)_i - z_i } \right)}
	\\
	&= \prod_{i=1}^{n}  \exp \left( \frac{ \varepsilon}{\Delta f} \cdot
	\underbrace{
		\abs{ f(y)_i - z_i } - \abs{ f(x)_i - z_i }
	}_\text{Apply Theorem \ref{theorem:triangle.ineq}}
    \right)
	\\
	&\leq \prod_{i=1}^{n}  \exp \left( \frac{ \varepsilon}{\Delta f} \cdot \abs{ f(y)_i - f(x)_i } \right) 
	\\
	&= \exp \left( \frac{ \varepsilon}{\Delta f}
	\cdot \underbrace{\sum_{i=1}^{n} \abs{ f(y)_i - f(x)_i }}_\text{Def.~of $\ell_1$ norm} 
	\right)
	\\
	&= \exp \left( \frac{ \varepsilon}{\Delta f} \cdot \underbrace{\norm{f(y) - f(x)}_1}_{\leq \Delta f \quad \text{Corollary \ref{corollary:sensitivity}}} \right) \label{eq:laplace.proof.key-part}
	\\
	&\leq \exp \left( \frac{ \varepsilon}{\Delta f} \cdot  \Delta f \right)
	\\
	&= \exp(\varepsilon)
	\end{align}
	which is what we wanted. By symmetry we get the proof for
	$\frac{p(M_L(y, f, \varepsilon) = z)}{p(M_L(x, f, \varepsilon) = z)} \leq \exp(\varepsilon)$.
	
\end{proof}

\section{Proof of Theorem \ref{theorem.krishna.sensitivity.disprove}}
\label{app:proof.of.sensitivity}

\begin{proof}
	The definition of sensitivity corresponds to the maximum $\ell_1$ distance of any two vectors $\mathbb{R}^n$ from the range of $f$. As Eq.~\ref{eq:f} bounds all vectors to their $\ell_2$ (Euclidean) norm, we want to find the distance between two opposing points on an $n$-dimensional sphere that have maximal $\ell_1$ distance.
	
	Let $n \in \mathbb{N} > 0$ be the number of dimension and $C \in \mathbb{R}$ a positive constant. We solve the following optimization problem

	\begin{maxi*}
		{x_1,\dots,x_n}{f(x_1, \dots, x_n) = \abs{x_1} + \dots + \abs{x_n}}{}{}
		\addConstraint{\sqrt{x_1^2 + \dots + x_n^2}}{=C}{}
	\end{maxi*}
	
	First, we can get rid of the absolute values in $f(x_1, \dots, x_n)$ as the maximums will be symmetric, i.e. $\max (\abs{a} + \abs{b}) = \max (\abs{-a} + \abs{-b})$. 
	
	Using Lagrange multipliers, we define the constraints as
	
	$$
	g(x_1, \dots, x_n) = \sqrt{x_1^2 + \dots + x_n^2} - C = 0,
	$$
	hence
	$$
	\begin{aligned}
	\mathcal{L}(x_1, \dots, x_n, \lambda) &= f(x_1, \dots, x_n) \\ & \quad + \lambda \cdot g(x_1, \dots, x_n) \\
	&= x_1 + \dots + x_n + \\ & \quad \lambda \sqrt{x_1^2 + \dots + x_n^2} - \lambda C
	\end{aligned}
	$$
	
	The gradient $\nabla_{x_1, \dots, x_n,\lambda} \mathcal{L}(x_1, \dots, x_n, \lambda)$ is
	
	$$
	\begin{aligned}
	\left( \pdv{\mathcal{L}}{x_1} , \dots, \pdv{\mathcal{L}}{x_n}, \pdv{\mathcal{L}}{\lambda} \right) = \\
	\left( \frac{x_1 \lambda}{\sqrt{x_1^2 + \dots + x_n^2}} + 1, \dots, \right. \\
	\left. \frac{x_n \lambda}{\sqrt{x_1^2 + \dots + x_n^2}} + 1, \sqrt{x_1^2 + \dots + x_n^2} - C \right)
	\end{aligned}
	$$
	
	Solve $\nabla_{x_1, \dots, x_n,\lambda} \mathcal{L}(x_1, \dots, x_n, \lambda) = 0$ by the following system of $n + 1$ equations

	$$
	\begin{aligned}
	\frac{x_1 \lambda}{\sqrt{x_1^2 + \dots + x_n^2}} + 1 &= 0 \\
	\dots &= 0 \\
	\frac{x_n \lambda}{\sqrt{x_1^2 + \dots + x_n^2}} + 1 &= 0 \\
	\sqrt{x_1^2 + \dots + x_n^2} - C &= 0
	\end{aligned}
	$$
	
	From the first $n$ expressions we get
	
	$$
	\begin{aligned}
	\lambda &= - \frac{\sqrt{x_1^2 + \dots + x_n^2}}{x_1} = \dots = \\
	&= - \frac{\sqrt{x_1^2 + \dots + x_n^2}}{x_n}, 
	\end{aligned}
	$$
	hence $x_1 = x_2 = \dots = x_n$. Plugging into the last term we obtain
	
	\begin{equation}
	\label{eq:maximum-norm-result}
	x_1 = x_2 = \dots = x_n = \frac{C}{\sqrt{n}}
	\end{equation}
	
	Geometrically, $x_i$ corresponds to the size of an edge of a hypercube embedded into a hypersphere of radius $C$.
	
	Now let $\mathbf{x}, \mathbf{x'} \in \mathbb{R}^n$ such that they have maximum $\ell_1$ norm (Eq.~\ref{eq:maximum-norm-result}) and their $\ell_2$ norm is $C$ (that is the output of function $f$ after clipping in Eq.~\ref{eq:f})
	
	$$
	\begin{aligned}
	\mathbf{x} &= \left(- \frac{C}{\sqrt{n}}, \dots, - \frac{C}{\sqrt{n}} \right), \\
	\mathbf{x'} &= \left( \frac{C}{\sqrt{n}}, \dots,  \frac{C}{\sqrt{n}} \right)
	\end{aligned}
	$$
	
	Then their $\ell_1$ distance is
	
	\begin{equation}
	\begin{aligned}
	\norm{\mathbf{x}- \mathbf{x'}}_1 &= \sum_{i = 1}^{n} \abs{- \frac{C}{\sqrt{n}} -\frac{C}{\sqrt{n}}} \\
	&= n \cdot \left( \frac{2C}{\sqrt{n}} \right) = 2C\sqrt{n}
	\end{aligned}
	\end{equation}
	
\end{proof}

\section{Potential attacks}
\label{sec:potential.attacks}

Here we only sketch a potential attack on a single individual's privatized output $\mathbf{v}$. We do not speculate on the actual feasibility as differentiall privacy operates with the worst case scenario, that is the theoretical possibility that the adversary has unlimited compute power and unlimited background knowledge. However, real life examples show that anything less protective than DP can be attacked and it is mostly a matter of resources.\footnote{Diffix, a EU-based company, claimed their system is a better alternative to DP but did not provide formal guarantees for such claims. A paper from \citet{Gadotti.et.al.2019.USENIX} was a bitter lesson for Diffix, as it shows a successful attack. The bottom line is that without formal guarantees, it is impossible to prevent any future attacks.}

We expect to have access to the trained ADePT autoencoder as well as the ATIS corpus (without the single individual whose value we try to infer, to be fair). We would need to find the privatized latent vector of $\mathbf{v}$, that is $\mathbf{r'}$, which could be possible by exploiting and probing the model. Second, by employing a brute-force attack, we can train a LM on ATIS to generate a feasible search space of input utterances, project them to the latent space, and explore the neighborhood of $\mathbf{r'}$. This would drastically reduce the search space. Then, depending on the geometric properties of that latent space, it might be the case that `similar' utterances are closer to each other, increasing the probability of finding a similar utterance which might be a `just good enough' approximation for the adversary.

\end{document}